\pdfoutput=1

\documentclass[11pt]{article}

\usepackage[]{ACL2023}

\usepackage{times}
\usepackage{latexsym}

\usepackage[T1]{fontenc}
\usepackage[utf8]{inputenc}
\usepackage{microtype}
\usepackage{inconsolata}

\usepackage{soul}
\usepackage{url}
\usepackage{hyperref}
\usepackage[]{caption}
\usepackage{graphicx}
\usepackage{amsmath}
\usepackage{amsfonts}
\usepackage{booktabs}
\usepackage{floatrow}
\usepackage{subcaption} 

\usepackage{enumitem}
\usepackage{makecell}
\usepackage{multicol}
\usepackage{multirow}
\usepackage{todonotes}
\usepackage{xcolor}         
\usepackage{amsmath,amsthm}
\usepackage{apptools}
\AtAppendix{\counterwithin{proposition}{section}}
\AtAppendix{\counterwithin{lemma}{section}}

\renewcommand{\Box}{\operatorname{Box}}
\newcommand{\QA}{\operatorname{QA}}

\newtheorem{proposition}{Proposition}[]

\newtheorem{lemma}{Lemma}[]
\newtheorem{definition}{Definition}[]

\title{Shrinking Embeddings for Hyper-Relational Knowledge Graphs}

\author{%
    Bo Xiong\thanks{Correspondence to \url{bo.xiong@ipvs.uni-stuttgart.de}}, \\
    University of Stuttgart\\
   \And
    Mojtaba Nayyeri\\
    University of Stuttgart\\
     \And
    Shirui Pan \\
    Griffith University \\
    \And
    Steffen Staab\\
    University of Stuttgart \\
    University of Southampton \\
}

\begin{document}
\maketitle
\begin{abstract}
Link prediction on knowledge graphs (KGs) has been extensively studied on binary relational KGs, wherein each fact is represented by a triple. 
A significant amount of important knowledge, however, is represented by hyper-relational facts where each fact is composed of a primal triple and a set of qualifiers comprising a key-value pair that allows for expressing more complicated semantics.
Although some recent works have proposed to embed hyper-relational KGs, these methods fail to capture essential inference patterns of hyper-relational facts such as qualifier monotonicity, qualifier implication, and qualifier mutual exclusion, limiting their generalization capability. 
To unlock this, we present \emph{ShrinkE}, a geometric hyper-relational KG embedding method aiming to explicitly model these patterns. ShrinkE models the primal triple as a spatial-functional transformation from the head into a relation-specific box. 
Each qualifier ``shrinks'' the box to narrow down the possible answer set and, thus, realizes qualifier monotonicity. 
The spatial relationships between the qualifier boxes allow for modeling core inference patterns of qualifiers such as implication and mutual exclusion. Experimental results demonstrate ShrinkE's superiority on three benchmarks of hyper-relational KGs. 
\end{abstract}

\section{Introduction}
Link prediction on knowledge graphs (KGs) is a central problem for many KG-based applications \citep{zhang2016collaborative,lukovnikov2017neural,DBLP:conf/sigir/jiayinglu,DBLP:conf/semweb/XiongPTNS22,DBLP:conf/wsdm/00010ZZMHK22}. 
Existing works \cite{DBLP:conf/iclr/SunDNT19,DBLP:conf/nips/BordesUGWY13} have mostly studied link prediction on binary relational KGs, wherein each fact is represented by a triple, e.g., (\emph{Einstein}, \emph{educated\_at}, \emph{University of Zurich}). 
In many popular KGs such as Freebase \citep{bollacker2007freebase}, however, a lot of important knowledge is not only expressed in triple-shaped facts, but also via facts about facts, which taken together are called hyper-relational facts. 
For example, ((\emph{Einstein}, \emph{educated\_at}, \emph{University of Zurich}), \{(\emph{major}:\emph{physics}), (\emph{degree}:\emph{PhD})\}) is a hyper-relational fact, where the primary triple (\emph{Einstein}, \emph{educated\_at}, \emph{University of Zurich}) is contextualized by a set of key-value pairs \{(\emph{major}:\emph{physics}),(\emph{degree}:\emph{PhD})\}. 
Like much other related work, we follow the terminology established for Wikidata \citep{vrandevcic2014wikidata} and use the term \emph{qualifiers} to refer to the key-value pairs.\footnote{Synonyms include statement-level \emph{metadata} in RDF-star \citep{rdf-star} and triple \emph{annotation} in provenance communities \citep{green2007provenance}.} 
The qualifiers play crucial roles in avoiding ambiguity issues. For instance, \emph{Einstein} was \emph{educated\_at} several universities and the qualifiers for \emph{degree} and \emph{major} help distinguish them. 
\par
In order to predict links in hyper-relational KGs, pioneering works represent each hyper-relational fact as either an $n$-tuple in the form of $r(e_1,e_2,\cdots,e_n)$ \citep{DBLP:conf/ijcai/WenLMCZ16,DBLP:conf/www/ZhangLMM18, DBLP:conf/ijcai/FatemiTV020,DBLP:conf/www/0016Y020,boxE} or a set of key-value pairs in the form of $\{(k_i:v_i)\}_{i=1}^m$ \citep{DBLP:conf/www/GuanJWC19,guan2021link,DBLP:conf/www/LiuYL21}. However, these modelings lose key structure information and are incompatible with the RDF-star schema \cite{rdf-star} used by modern KGs, where both primal triples and qualifiers constitute the fundamental data structure. 
Recent works \citep{DBLP:conf/acl/GuanJGWC20,DBLP:conf/www/RossoYC20} represent each hyper-relational fact as a primary triple coupled with a set of qualifiers that are compatible with RDF-star standards \citep{rdf-star}. 
Link prediction is then achieved by modeling the validity of the primary triple and its compatibility with each annotated qualifier \citep{DBLP:conf/acl/GuanJGWC20, DBLP:conf/www/RossoYC20}. More complicated graph encoders and decoders \citep{DBLP:conf/emnlp/GalkinTMUL20,DBLP:journals/corr/abs-2104-08167,DBLP:conf/acl/WangWLZ21,DBLP:journals/corr/abs-2208-14322} are proposed to further boost the performance. However, they require a relatively huge number of parameters that make them prone to overfitting. 

To encourage generalization capability, KG embeddings should be able to model inference patterns, i.e., specifications of logical properties that may exist in KGs, which, if learned, empowers further principled inferences \cite{boxE}. 
This has been extensively studied for binary relational KG embeddings \cite{DBLP:conf/icml/TrouillonWRGB16,DBLP:conf/iclr/SunDNT19} but ignored for hyper-relational KGs in which not only primal triples but also qualifiers matter. 
One of the most important properties is qualifier monotonicity.
Given a query, the answer set shrinks or at least does not expand as more qualifiers are added to the query expression. 
For example, a query $\left(\emph{Einstein}, \emph{educated\_at}, ?x\right)$ with a variable $?x$ corresponds to two answers $\{\emph{University of Zurich}, \emph{ETH Zurich}\}$, but a query $\left(\left(\emph{Einstein}, \emph{educated\_at},?x\right),\{\left(\emph{degree}:\emph{B.Sc.}\right)\} \right)$ extended by a qualifier for \emph{degree} will only respond with $\{\emph{ETH Zurich}\}$. Besides, different qualifiers might form logical relationships that the model must respect during inference including qualifier implication (e.g., adding a qualifier that is implicitly implied in the existing qualifiers does not change the truth of a fact) and qualifier mutual exclusion (e.g., adding any two mutually exclusive qualifiers to a fact leads to a contradiction). 
\par
In light of this, we propose ShrinkE, a hyper-relational embedding model that allows for modeling these inference patterns. 
ShrinkE embeds each entity as a point and models a primal triple as a spatio-functional transformation from the head entity to a relation-specific box that entails the possible tails. Each qualifier is modeled as a shrinking of the primal box to a qualifier box. 
The shrinking of boxes simulates the ``monotonicity'' of hyper-relational qualifiers, i.e., attaching qualifiers to a primal triple may only narrow down but never enlarges the answer set. 
The plausibility of a given fact is measured by a point-to-box function that judges whether the tail entity is inside the intersection of all qualifier boxes.
Moreover, since each qualifier is associated with a box, the spatial relationships between the qualifier boxes allow for modeling core inference patterns such as qualifier implication and mutual exclusion. We theoretically show the capability of ShrinkE on modeling various inference patterns including (fact-level) monotonicity, triple-level, and qualifier-level inference patterns. Empirically, ShrinkE achieves competitive performance on three benchmarks.

\begin{figure*}[t!]
    \centering
    \subfloat[\centering ]{{\includegraphics[width=.46\linewidth]{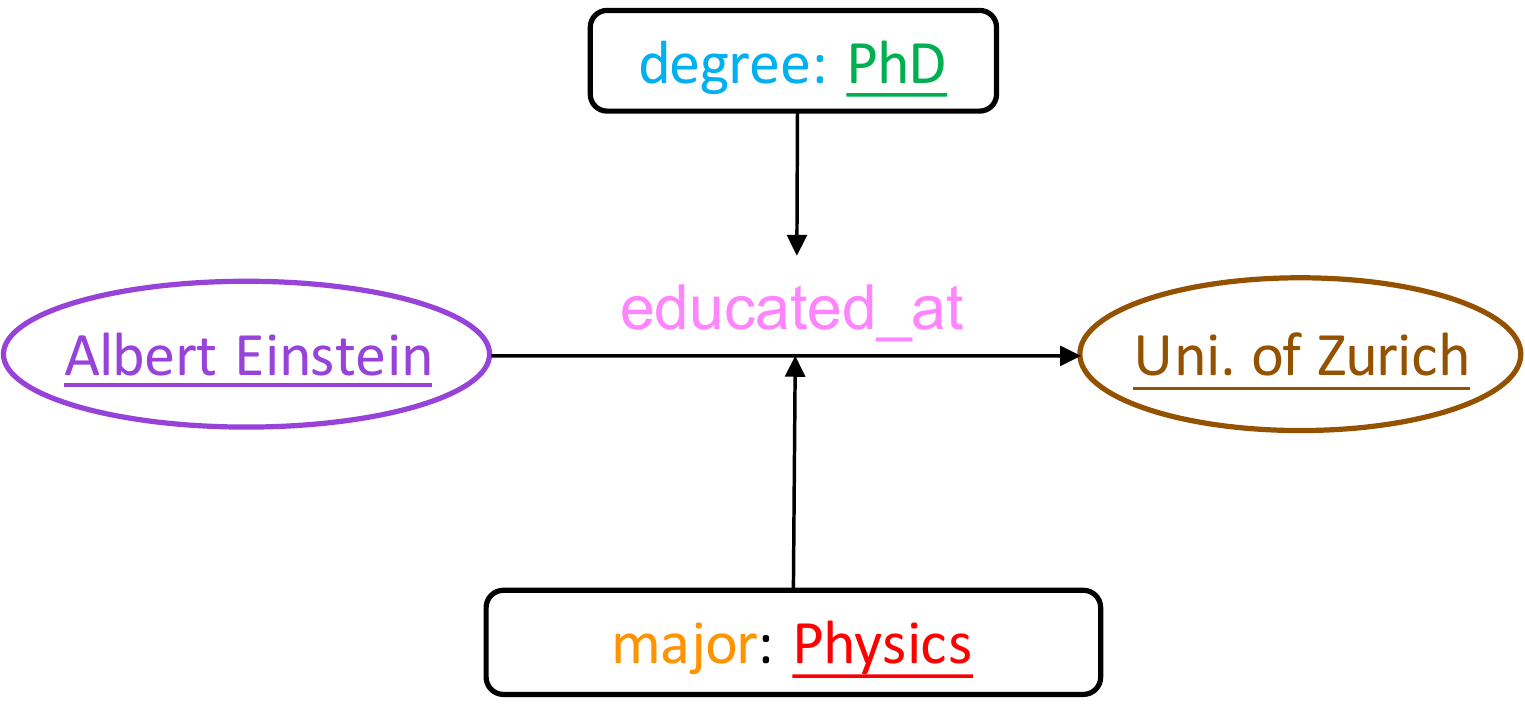}}}
    \subfloat[\centering ]{{\includegraphics[width=.54\linewidth]{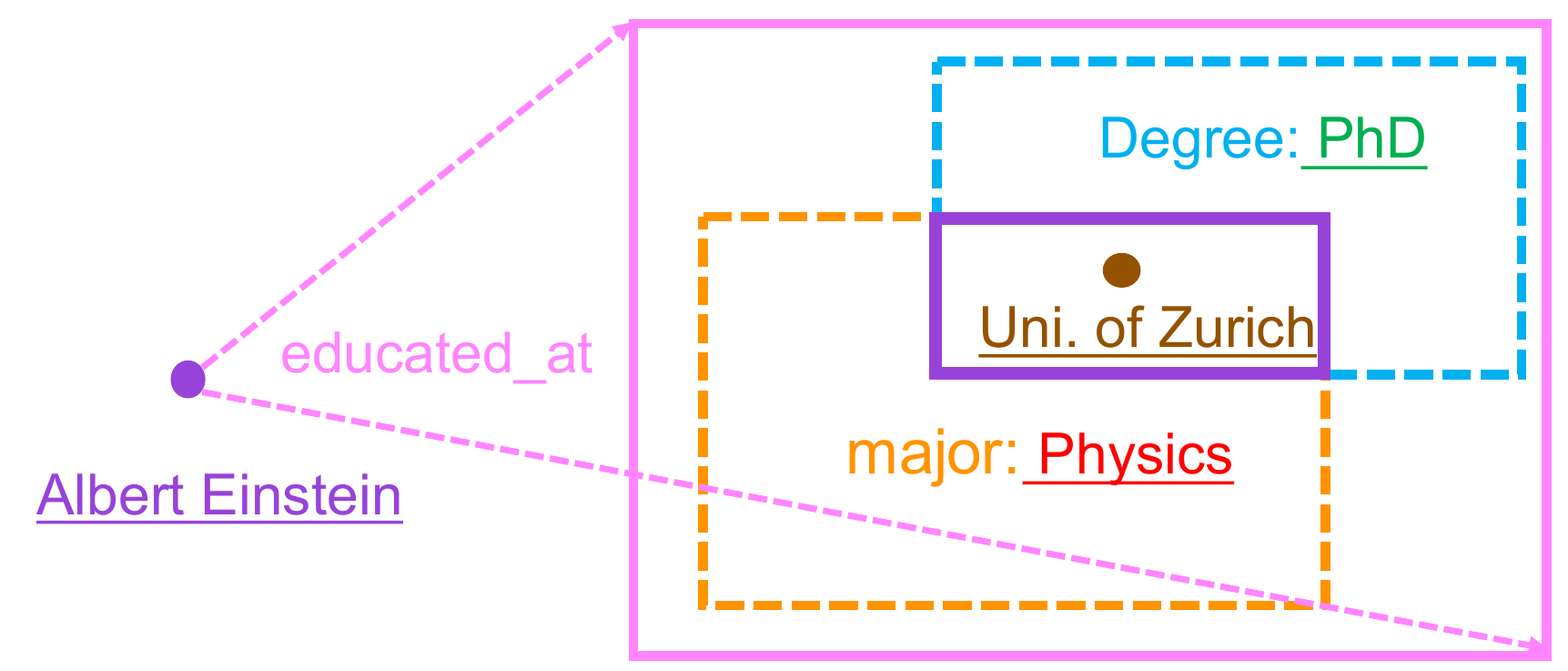} }}
    \caption{An illustration of the proposed idea. (a) A hyper-relational fact is composed of a primal triple and two key-value qualifiers, in which entities (values) are \underline{underlined} while relations (keys) are not.  (b) An illustration of the proposed hyper-relational KG embedding model ShrinkE. ShrinkE models the primal triple as a relation-specific transformation from the head entity to a query box (\emph{purple}) that entails the possible answer entities. Each qualifier is modeled as a shrinking of the query box (\emph{orange} and \emph{cyan}) such that the shrinking box is a subset of the query box. The shrinking of the box can be viewed as a geometric interpretation of the monotonicity assumption that we follow. 
    The final answer entities are supposed to be in the intersection box of all shrinking boxes. 
    } 
    \label{fig:example_bigmap}
\end{figure*}

\section{Related Work}

Related works on hyper-relational KG embeddings can be categorized by their representations of facts. Prominent representations include tuple, key-value pairs, and triple+key-value pairs. 

\paragraph{Tuple based} 
Pioneering works view a hyper-relational fact as an $n$-tuple, a.k.a.\ n-ary fact, consisting of a single abstract relation $r$ and its $n$ values, i.e., $r(v_1,v_2,\cdots, v_n)$. 
Functional models represent the tuple-based facts by functional mapping. For example, m-TransH \citep{DBLP:conf/ijcai/WenLMCZ16}, a generalization of TransH \citep{DBLP:conf/aaai/WangZFC14} to hyper-relational facts, projects all entities onto a relation-specific hyperplane and measures the plausibility as the weighted sum of projected embeddings. 
RAE \citep{DBLP:conf/www/ZhangLMM18} improves m-TransH by further modeling the \emph{relatedness} of values. Multilinear models generalize bilinear models to hyper-relational facts via multi-linear products. 
For example, HsimplE \citep{DBLP:conf/ijcai/FatemiTV020}, m-CP \citep{DBLP:conf/ijcai/FatemiTV020}, and GETD \citep{DBLP:conf/www/0016Y020} generalize SimplE \citep{DBLP:conf/nips/Kazemi018}, Canonical Polyadic (CP) decomposition \cite{trouillon2017knowledge}, and TuckER \citep{DBLP:conf/emnlp/BalazevicAH19}, respectively. GETD only applies to KGs with single-arity relations \citep{liu2021role} and S2S \citep{DBLP:conf/www/DiYC21} extends it to support mixed arity facts. HypE \citep{DBLP:conf/ijcai/FatemiTV020} encodes hyper-relational facts by positional convolutional filters and evaluates the facts' plausibility using the multilinear product. However, these models ignore the semantics of relations and loosely represent a combination of all relations of the original fact \citep{DBLP:conf/emnlp/GalkinTMUL20}. 

\paragraph{Key-value pairs} NaLP \citep{DBLP:conf/www/GuanJWC19} view each hyper-relational fact as a set of key-value pairs, i.e., $\{(k_i:v_i)\}_{i=1}^m$. Convolutional networks are employed to encode the key-value pairs, followed by a multi-layer perceptron (MLP) that measures the compatibility between the key and its values. 
RAM \citep{liu2021role} further models the relatedness between different keys and the relatedness between a key and all involved values. 
NaLP+ \citep{guan2021link} improves NaLP by considering type information. 
However, the key-value-based modeling treats all key-value pairs equally and does not distinguish primal triples from qualifiers.  

\paragraph{Triple+key-value pairs} NeuInfer \citep{DBLP:conf/acl/GuanJGWC20} and HINGE \citep{DBLP:conf/www/RossoYC20} represent a hyper-relational fact as a primary triple combined with a set of the key-value form of qualifiers, i.e., $\left(\left(h,r,t\right), \left\{\left(k_i:v_i\right)\right\} _{i=1}^m \right)$, which is compatible with the RDF-star standard \citep{delva2021rml} used in modern KGs.  
Both methods adopt neural networks to obtain the fact validity by measuring the validity of the primary triple and its compatibility with each qualifier. 
NeuInfer applies MLP while HINGE uses a convolutional network as an encoder. 
StarE \citep{DBLP:conf/emnlp/GalkinTMUL20} leverages a message passing network, CompGCN \citep{DBLP:conf/iclr/VashishthSNT20}, as an encoder to obtain the relation and entity embeddings, which are then fed into a transformer decoder to obtain the validity of facts. Hy-Transformer \citep{DBLP:journals/corr/abs-2104-08167}, GRAN \citep{DBLP:conf/acl/WangWLZ21} and QUAD \citep{DBLP:journals/corr/abs-2208-14322} further improve it with alternative designs of encoders and via auxiliary training tasks. Relatively, these models, though useful, require a large number of parameters and are prone to overfitting. 


\section{Preliminaries}
We view a hyper-relational fact in the form of a primal triple coupled with a set of qualifiers. 

\begin{definition}[Hyper-relational fact]
Let $\mathcal{E}$ and $\mathcal{R}$ denote the sets of entities and relations, respectively. A hyper-relational fact $\mathcal{F}$ is a tuple $(\mathcal{T}, \mathcal{Q})$, where $\mathcal{T}=(h, r, t),\ h, t \in \mathcal{E}, r \in \mathcal{R} $ is a primal triple and $\mathcal{Q}=\left\{\left(k_{i}: v_{i}\right)\right\}_{i=1}^{m} \ k_{i} \in \mathcal{R}, v_{i} \in \mathcal{E}$ is a set of qualifiers.
We call the number of involved entities in $\mathcal{F}$, i.e., $(m+2)$, the arity of the fact. 
\end{definition}
A hyper-relational fact reduces to a triple/binary fact when $m=0$. When $m>0$, each qualifier can be viewed as an auxiliary description that contextualizes or specializes the semantics of the primal triple. 
In typical open-world settings, facts with the same primal triple might have different numbers of qualifiers. 
To characterize this property, we introduce the concepts of partial fact and qualifier monotonicity in hyper-relational KGs.

\begin{definition}[Partial fact \citep{DBLP:conf/acl/GuanJGWC20}]
Given two facts $\mathcal{F}_1=\left(\mathcal{T},\mathcal{Q}_1\right)$ and $\mathcal{F}_2=\left(\mathcal{T},\mathcal{Q}_2\right)$ that share the same primal triple. We call $\mathcal{F}_1$ a partial fact of $\mathcal{F}_2$ iff $\mathcal{Q}_1 \subseteq \mathcal{Q}_2$. 
\end{definition}

In this paper, we follow the monotonicity assumption by restricting the model to respect the monotonicity property.\footnote{
Some kinds of qualifiers may represent semantically opaque contexts. For instance, ((\emph{Crimea}, \emph{belongs\_to}, \emph{Russia}), \{(\emph{said\_by}, \emph{Putin})\}) does not imply the primary triple and should therefore be excluded. } For this purpose, we consider the monotonicity of query and inference.

\begin{definition}[Qualifier monotonicity]
Let $\QA(\cdot)$ denote a query answering model taking a query and a KG as input and outputting the set of answer entities. Given any pair of queries $q_1=\left(\left(h,r,x?\right), \mathcal{Q}_1\right)$ and $q_2=\left(\left(h,r,x?\right), \mathcal{Q}_2\right)$ that share the same primal triple and $\mathcal{Q}_1 \subseteq \mathcal{Q}_2$, qualifier monotonicity is given iff,
\begin{equation}
\QA(q_2; \operatorname{KG}) \subseteq \QA(q_1;\operatorname{KG}).
\end{equation}
\end{definition}
Qualifier monotonicity implies that attaching any qualifiers to a query does not enlarge the answer set of the possible tail entities, and inversely, removing the qualifiers from a query can only return more possible tail entities. This implies that if a fact is true, then all its partial facts must also be true (a.k.a. weakening of inference rule), i.e.,
\begin{equation}
    \left(\mathcal{T},\mathcal{Q}_1\right)  \wedge (\mathcal{Q}_2 \subseteq \mathcal{Q}_1) \rightarrow \left(\mathcal{T},\mathcal{Q}_2\right).
    \label{eq:momonotonicity}
\end{equation}

\section{Shrinking Embeddings for Hyper-Relational KGs }

We aim to design a scoring function $f(\cdot)$ taking the embeddings of facts as input so that the output values respect desired logical properties. 
To this end, we introduce primal triple embedding and qualifier embedding, respectively.

\subsection{Primal Triple Embedding}
We represent each entity as a point $\mathbf{e} \in \mathbb{R}^d$. Each primal relation $r$ is modeled as a spatio-functional transformation $\mathcal{B}_r:\mathbb{R}^d \rightarrow \Box(d)$ that maps the head $\mathbf{e}_h \in \mathbb{R}^d$ to a $d$-dimensional box in $\Box(d)$ with $\Box(d)$ being the set of boxes in $\mathbb{R}^d$. 
Each box can be parameterized by a lower left point $\mathbf{m} \in \mathbb{R}^d$ and an upper right point $\mathbf{M} \in \mathbb{R}^d$, given by
\begin{equation}
\begin{split}
    &\Box^d(\mathbf{m}, \mathbf{M}) = \\
    &\{\mathbf{x} \in \mathbb{R}^d \mid \mathbf{m}_i \leq \mathbf{x}_i \leq \mathbf{M}_i, \: i=1,\cdots,d\}.
\end{split}
\end{equation}
We leave the superscript of $\Box^d$ away if it is clear from context. 
and call the transformed box a query box. Intuitively, all points in the query box correspond to the possible answer tail entities. Hence, the query box can be viewed as a geometric embedding of the answer set. 
Note that a query could result in an empty answer set. In order to capture such property, we do not exclude empty boxes that correspond to queries with empty answer set. Empty boxes are covered by the cases where there exists a dimension $i$ such that $\mathbf{m}_i \geq \mathbf{M}_i$. 

\paragraph{Point-to-box transform}
The spatio-functional point-to-box transformation $\mathcal{B}$ is composed of a relation-specific point transformation $\mathcal{H}_r:  \mathbb{R}^d \rightarrow \mathbb{R}^d$ that transforms the head point $\mathbf{e}_h$ to a new point, and a relation-specific spanning that spans the transformed point to a box, formally given by
\begin{equation}
    \mathcal{B}_r(\mathbf{e}_h) = \Box( \mathcal{H}_r(\mathbf{e}_h) - \tau(\mathbf{\boldsymbol\delta}_r),  \mathcal{H}_r(\mathbf{e}_h) + \tau(\mathbf{\boldsymbol\delta}_r)),
    \label{eq:span}
\end{equation}
where $\mathbf{\boldsymbol\delta}_r \in \mathbb{R}^n$ is a relation-specific spanning/offset vector, and $\tau_{t}(\mathbf{x})=t \log \left(1+e^{\mathbf{x}/t}\right)$ with $t$ being a temperature hyperparameter, is a \emph{softplus} function that enforces the spanned box to be non-empty. 

The point transformation function $\mathcal{H}_r$ could be any functions that are used in other KG embedding models such as translation used in TransE \cite{DBLP:conf/nips/BordesUGWY13} and rotations used in RotatE \cite{DBLP:conf/iclr/SunDNT19}. 
Hence, our model is highly flexible and effective at embedding primal triples. To allow for capturing multiple triple-level inference patterns such as symmetry, inversion, and composition, we combine translation and rotation, and formulate $\mathcal{H}_r$ as
\begin{equation}
    \mathcal{H}_r(\mathbf{e}_h) = \Theta_{r} \mathbf{e}_h + \mathbf{b}_{r}
\end{equation}
where $\Theta_{r}$ is a rotation matrix and $\mathbf{b}_{r}$ is a translation vector.
We parameterize the rotation matrix by a block diagonal matrix $ \Theta_{r} =\operatorname{diag}\left(\mathbf{G}\left(\theta_{r, 1}\right), \ldots, \mathbf{G}\left(\theta_{r, \frac{d}{2}}\right)\right)$, 
where 
\begin{equation}
    \mathbf{G}(\theta)=\left[\begin{array}{cc}
    \cos (\theta) &  \sin (\theta) \\
    \sin (\theta) &  \cos (\theta)
\end{array}\right].
\end{equation}
\paragraph{Point-to-box distance}
The validity of a primal triple $(h,r,t)$ is then measured by judging whether the tail entity point $\mathbf{e}_t$ is geometrically inside of the query box. 
Given a query box $\Box^n(\mathbf{m}, \mathbf{M})$ and an entity point $\mathbf{e} \in \mathbb{R}^{d}$, we denote the center point as $\mathbf{c}=\frac{\mathbf{m} + \mathbf{M}}{2}$. Let $|\cdot|$ denote the L1 norm and $\max()$ denote an element-wise maximum operation. The point-to-box distance is given by
\begin{equation}
\begin{split}
    &D(\mathbf{e}, \Box(\mathbf{m}, \mathbf{M})) =  \frac{|\mathbf{e}-\mathbf{c}|_1}{|\max(\mathbf{0}, \mathbf{M}-\mathbf{m})|_1} \\
    &+ \left(|\mathbf{e}-\mathbf{m}|_1 + |\mathbf{e}-\mathbf{M}|_1  - |\max(\mathbf{0}, \mathbf{M}-\mathbf{m})|_1 \right)^2.
\end{split}
\end{equation}
Fig. \ref{fig:point2box} visualizes the distance function. Intuitively, in cases where the point is in the query box, the distance grows relatively slowly and inversely correlates with the box size. In cases where the point is outside the box, the distance grows fast. 

\begin{figure}[t!]
    \centering
    \includegraphics[width=\linewidth]{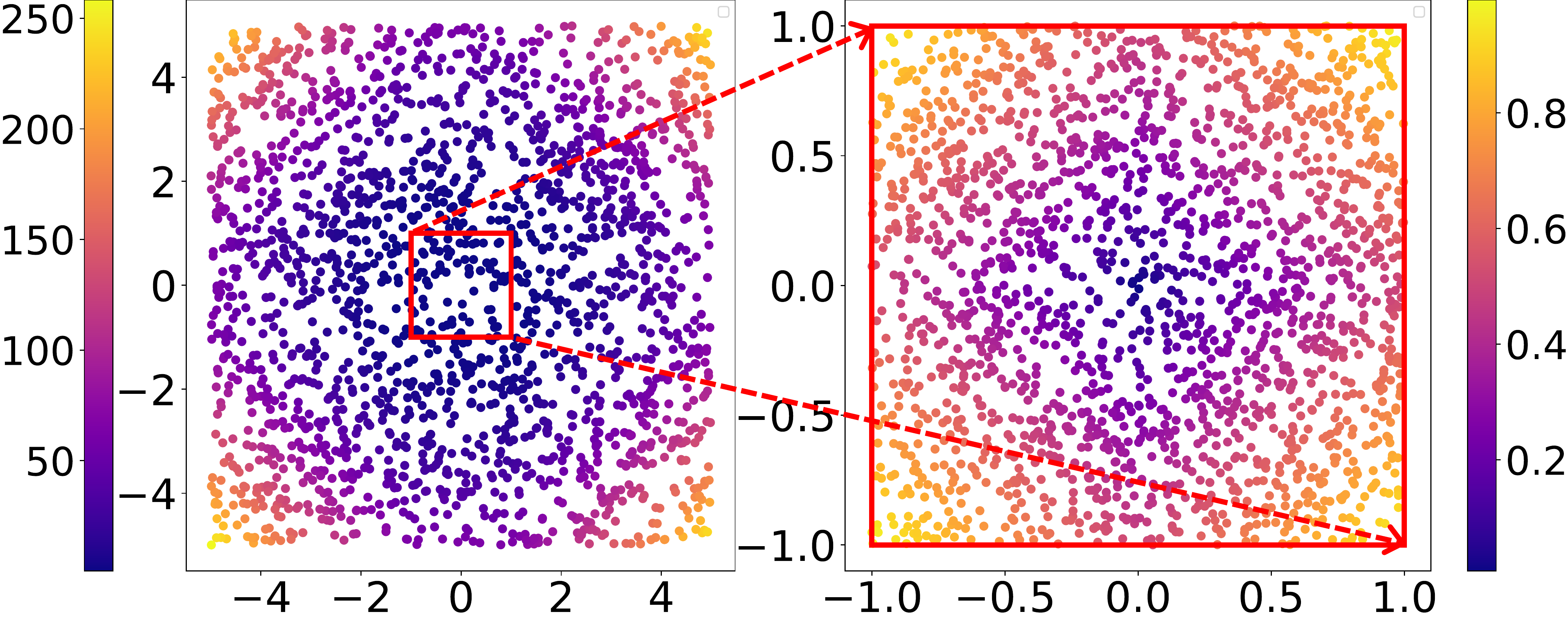}
    \caption{An illustration of the point-to-box distance. The distance (visualized by color maps) grows slowly when the point is inside of the box (\emph{right}) while growing faster when the point is outside of the box (\emph{left}). } 
    \label{fig:point2box}
\end{figure}

\subsection{Qualifier Embedding}

Conceptually, qualifiers add information to given primary facts potentially allowing for additional inferences, but never for the retraction of inferences, reflecting the monotonicity of the representational paradigm. 
Corresponding to the non-declining number of inferences, the number of possible models for this representation shrinks, which can be intuitively reflected by a reduced size of boxes incurred by adding qualifiers. 

\paragraph{Box Shrinking}
To geometrically mimic this property in the embedding space, we model each qualifier $(k:v)$ as a "shrinking" of the query box.
Given a box $\Box(\mathbf{m}, \mathbf{M})$, a shrinking is defined as a box-to-box transformation $\mathcal{S}: \Box \rightarrow \Box$ that potentially shrinks the volume of the box while not moving the resulting box outside of the source box. Let $\mathbf{L} = \left(\mathbf{M}-\mathbf{m}\right)$ denote the side length vector, box shrinking is defined by
\begin{equation}\normalsize
\begin{split}
&\mathcal{S}_{r,k,v}\left( \Box\left(\mathbf{m}, \mathbf{M}\right) \right) =\\
&\Box \left( \mathbf{m} + 
\sigma\left(\mathbf{s}_{r,k,v}\right) \odot \mathbf{L} ,
\mathbf{M} - \sigma\left(\mathbf{S}_{r,k,v}\right) \odot \mathbf{L} \right),
\end{split}
\end{equation}
where $\mathbf{s}_{r,k,v} \in \mathbb{R}^n$ and $\mathbf{S}_{r,k,v} \in \mathbb{R}^n$ are the "shrinking" vectors for the lower left corner and the upper right corner, respectively. $\sigma$ is a \emph{sigmoid} function and $\odot$ is element-wise vector multiplication. 
The resulting box, including the case of empty box, is always inside the query box, i.e.,  $\mathcal{S}( \Box^n(\mathbf{m}, \mathbf{M}) ) \subseteq \Box^n(\mathbf{m}, \mathbf{M})$, which exactly resembles the qualifier monotonicity.

We use $r,k,v$ as the indices of the shrinking vectors because the shrinking of the box should depend on the relatedness between the primal relation and the qualifier. 
For example, if a qualifier $({\emph{degree}:\emph{bachelor}})$ is highly related to the primal relation $\emph{educated\_at}$, the scale of the shrinking vectors should be small as it adds a weak constraint to the triple. 
If the qualifier is unrelated to the primal relation, e.g., $({\emph{degree}:\emph{bachelor}})$ and $\emph{born\_in}$, the shrinking might even enforce an empty box. 

To learn the shrinking vectors, we leverage an MLP layer that takes the primal relation and key-value qualifier as input and outputs the shrinking vectors defined by $\mathbf{s}_{r,k,v}, \mathbf{S}_{r,k,v} = \operatorname{MLP}\left(\operatorname{concat} \left( r_\theta, k_\theta, v_\theta \right) \right)$ where $r_\theta, k_\theta, v_\theta$ are the embeddings of $r,k,v$, respectively. 

\subsection{Scoring function and learning.}

\paragraph{Scoring function} 

The score of a given hyper-relational fact is defined by 
\begin{equation}
    f\left( \left( \left(h,r,t\right),\mathcal{Q}\right) \right) = D(\mathbf{e}_t, \Box_\mathcal{Q}(\mathbf{m}, \mathbf{M})),
\end{equation}
where $\Box_\mathcal{Q}(\mathbf{m}, \mathbf{M})$ denotes the target box that is calculated by the intersection of all shrinking boxes of the qualifier set $\mathcal{Q}$. 
The intersection of $n$ boxes can be calculated by taking the maximum of lower left points of all boxes and taking the minimum of upper right points of all boxes, given by
\begin{equation}\small
\begin{split}
    \mathcal{I}(\Box_1, \cdots, \Box_n)  = 
    \Box\left( \max_{i \in 1, \cdots,n} \mathbf{m}_i, \min_{i \in 1, \cdots,n} \mathbf{M}_i \right).
\end{split}
\end{equation}

Note that if there is no intersection between boxes, this intersection operation still works as it results in an empty box. The intersection of boxes is a permutation-invariant operation, implying that perturbing the order of qualifiers does not change the plausibility of the facts. 

\paragraph{Learning}
As a standard data augmentation strategy, we add reciprocal relations $\left(t^{\prime}, r^{-1}, h^{\prime}\right)$ for the primary triple in each hyper-relational fact.
For each positive fact in the training set, we generate $n_{\operatorname{neg}}$ negative samples by corrupting a subject/tail entity with randomly selected entities from $\mathcal{E}$. 
We adopt the cross-entropy loss to optimize the model via the Adam optimizer, which is given by
\begin{equation}\small
    \mathcal{L}=-\frac{1}{N} \sum_{i=1}^{N}\left(y_{i} \log \left(p_{i}\right)+ \sum_{i=1}^{n_{\operatorname{neg}}} \left(1-y_{i}\right) \log \left(1-p_{i}\right)\right),
\end{equation}
where $N$ denotes the total number of facts in the training set. $y_{i}$ is a binary indicator denoting whether a fact is true or not. $p_{i}=\sigma(f(\mathcal{F}))$ is the predicted score of a fact $\mathcal{F}$ with $\sigma$ being the \emph{sigmoid} function. 

\begin{table*}[ht]
    \centering
      \caption{Dataset statistics, where the columns indicate the number of all facts, hyper-relational facts with the number of qualifiers $m>0$, entities, relations, and facts in train/dev/test sets, respectively.}
      \resizebox{\linewidth}{!}{
    \begin{tabular}{cccccccccc}
\hline
& All facts & Higher-arity facts (\%) & Entities & Relations & Train & Dev & Test \\ 
\hline
JF17K & 100,947 & 46,320 (45.9\%) & 28,645 & 501 & 76,379 & – & 24,568 \\
WikiPeople & 382,229 & 44,315 (11.6\%) & 47,765 & 193 & 305,725 & 38,223 & 38,281 \\
WD50k & 236,507 & 32,167 (13.6\%) & 47,156 & 532 & 166,435 & 23,913 & 46,159 & \\
WD50K(33) & 102,107 & 31,866 (31.2\%) & 38,124 & 475 & 73,406 & 10,568 & 18,133\\
WD50K(66) & 49,167  & 31,696 (64.5\%) & 27,347 & 494 & 35,968 & 5,154 & 8,045\\
WD50K(100) & 31,314 & 31,314 (100\%)  & 18,792 & 279 & 22,738 & 3,279 & 5,297\\
\hline
    \end{tabular}}
    \label{tab:dataset}
\end{table*}

\section{Theoretical Analysis}
\label{sec:theorem}

Analyzing and modeling inference patterns is of great importance for KG embeddings because it enables generalization capability, i.e., once the patterns are learned, new facts that respect the patterns can be inferred. 
An inference pattern is a specification of a logical property that may exist in a KG, Formally, an inference pattern is a logical form $\psi \rightarrow \phi$ with $\psi$ and $\phi$ being the body and head, implying that if the body is satisfied then the head must also be satisfied. 

In this section, we analyze the theoretical capacity of ShrinkE for modeling inference patterns. All proofs of propositions are in Appendix \ref{app:proof}. 

\paragraph{Fact-level inference pattern (monotonicity)} The following proposition shows that ShrinkE is able to model monotonicity. 
\begin{proposition}
Given any two facts $\mathcal{F}_1=\left(\mathcal{T},\mathcal{Q}_1\right)$ and $\mathcal{F}_2=\left(\mathcal{T},\mathcal{Q}_2\right)$ where $\mathcal{Q}_2 \subseteq \mathcal{Q}_1$, i.e., $\mathcal{F}_2$ is a partial fact of $\mathcal{F}_1$, the output of the scoring function $f(\cdot)$ of ShrinkE satisfy the constraint $f(\mathcal{F}_2) \geq f(\mathcal{F}_1)$. 
\label{prop:qualifier_monotonicity}
\end{proposition}

\paragraph{Triple-level inference patterns} 
Prominent triple-level inference patterns include symmetry 
$(h,r,t) \rightarrow (t,r,h)$, anti-symmetry $(h,r,t) \rightarrow \neg (h,r,t)$, inversion $(h,r_1,t) \rightarrow (t,r_2,h)$, composition $(e_1,r_1,e_2) \wedge (e_2,r_2,e_3)  \rightarrow (e_1,r_3,e_3)$, relation implication $(h,r_1,t) \rightarrow (h,r_2,t)$, relation intersection $(h,r_1, t) \wedge (h, r_2, t) \rightarrow (h, r_3, t)$, and relation mutual exclusion $(h,r_1, t) \wedge (h, r_2, t) \rightarrow \bot $. All these triple-level inference patterns also exist in hyper-relational facts when their qualifiers are the same, e.g., hyper-relational symmetry means $\left(\left(h,r,t\right), \mathcal{Q}\right) \rightarrow \left(\left(t,r,h\right), \mathcal{Q}\right)$. 
Proposition \ref{prop:triple_inference_pattern} states that ShrinkE is able to infer all of them. 
\begin{proposition}
    ShrinkE is able to infer hyper-relational symmetry, anti-symmetry, inversion, composition, relation implication, relation intersection, and relation exclusion.  
\label{prop:triple_inference_pattern}
\end{proposition}

\paragraph{Qualifier-level inference pattern}
In hyper-relational KGs, inference patterns not only exist at the triple level but also at the level of qualifiers. 

\begin{definition}[qualifier implication]
Given two qualifiers $q_i$ and $q_j$, $q_i$ is said to imply $q_j$, i.e., $q_i \rightarrow q_j$ iff for any fact $\mathcal{F} = \left(\mathcal{T}, \mathcal{Q}\right)$, if attaching $q_i$ to $\mathcal{Q}$ results in a true (resp. false) fact, then attaching $q_j$ to $\mathcal{Q}\cup \left\{q_i\right\}$ also results in a true (resp. false) fact. Formally, $q_i \rightarrow q_j$ implies
\begin{equation}
    \forall \ \mathcal{T},\mathcal{Q}: \left(\mathcal{T}, \mathcal{Q}\cup \left\{q_i\right\}\right) \rightarrow \left(T, Q \cup \{q_i, q_j\}\right).
\end{equation}
\end{definition}

\begin{definition}[qualifier exclusion] 
Two qualifiers $q_i,q_j$ are said to be mutually exclusive iff for any fact $\mathcal{F} = \left(\mathcal{T}, \mathcal{Q}\right)$, 
by attaching $q_i,q_j$ to the qualifier set of $\mathcal{F}$, 
the new fact $\mathcal{F}^{\prime}=\left(\mathcal{T}, \mathcal{Q} \cup \left\{q_i,q_j\right\} \right)$ is false, meaning that they lead to a contradiction, i.e.,  $q_i \wedge q_j \rightarrow \bot$. Formally, $q_i \wedge q_j \rightarrow \bot$ implies
\begin{equation}
    \forall \ \mathcal{T},\mathcal{Q} \: : \left(T, Q \cup \left\{q_i,q_j\right\} \right) \rightarrow \bot 
\end{equation}
\end{definition}

Note that if two qualifiers $q_i,q_j$ are neither mutually exclusive nor forming implication pair, then $q_i,q_j$ are said to be overlapping, a state between implication and mutual exclusion. 
Qualifier overlapping, in our case, can be captured by box intersection/overlapping. 
Qualifier overlapping itself does not form any logical property in the form of $\psi \rightarrow \phi$. 
However, when involving three qualifiers and two of them overlap, qualifier intersection can be modeled. 

\begin{definition}[qualifier intersection] 

A qualifier $q_k$ is said to be an intersection of two qualifiers $q_i,q_j$ iff for any fact $\mathcal{F} = \left(\mathcal{T}, \mathcal{Q}\right)$, if attaching $q_i,q_j$ to $\mathcal{Q}$ results in a true (resp. false) fact, then by replacing $\{q_i,q_j\}$ with $q_k$, the truth value of the fact does not change. Namely, $q_i \wedge q_j \rightarrow q_k$ implies
\begin{equation}
    \forall \ \mathcal{T},\mathcal{Q}: \left(T, Q \cup \{q_i,q_j\}  \right) \rightarrow  \left(\mathcal{T}, \mathcal{Q} \cup \left\{q_k\right\} \right).
\end{equation}
\end{definition}

Apparently, qualifier intersection $ q_i \wedge q_j \rightarrow  q_k$ necessarily implies qualifier implications $q_i \rightarrow q_k$ and $q_j \rightarrow q_k$. 
Hence, qualifier intersection can be viewed as a combination of two qualifier implications, and this can be generalized to $ q_1 \wedge q_2 \wedge \cdots  \rightarrow  q_k$.
Proposition \ref{prop:qualifier_inference_pattern} shows that ShrinkE is able to infer qualifier implication, exclusion, and composition. 

\begin{proposition}
ShrinkE is able to infer qualifier implication, mutual exclusion, and intersection.
\label{prop:qualifier_inference_pattern}
\end{proposition}

\begin{table*}[]
    \centering
    \caption{Link prediction results on three benchmarks with the number in the parentheses denoting the ratio of facts with qualifiers. Baseline results are taken from \citet{DBLP:conf/emnlp/GalkinTMUL20}. }
    \resizebox{\linewidth}{!}{
    \begin{tabular}{lcccccccccccc}
    \hline 
    \multirow{2}{*}{Method} & \multicolumn{3}{c}{WikiPeople (2.6)}  & & \multicolumn{3}{c}{JF17K (45.9)} & & \multicolumn{3}{c}{WD50K (13.6)} \\
    \cline {2-4} \cline {6-8} \cline{10-12} & MRR & H@1 & H@10 & & MRR & H@ 1& H@ 10 & & MRR & H@ 1 & H@ 10 \\
    \hline 
    m-TransH & $0.063$ & $0.063$ & $0.300$ & & $0.206$ & $0.206$ & $0.463$ & & $-$ & $-$ & $-$\\
    RAE & $0.059$ & $0.059$ & $0.306$ & & $0.215$ & $0.215$ & $0.469$ & & $-$ & $-$ & $-$\\
    NaLP-Fix & $0.420$ & $0.343$ & $0.556$ & & $0.245$ & $0.185$ & $0.358$ & & 0.177 & 0.131 & 0.264 \\
    NeuInfer & $0.350$ & $0.282$ & $0.467$ & & $0.451$ & $0.373$ & $0.604$ & & $-$ & $-$ & $-$ \\
    HINGE & $0.476$ & $0.415$ & $0.585$ & & $0.449$ & $0.361$ & $0.624$ & & $0.243$ & $0.176$ & $0.377$ \\
    Transformer & 0.469 & 0.403 & 0.586 & & 0.512 & 0.434 & 0.665 & & 0.264 & 0.194 & 0.401 \\
    BoxE & 0.395 & 0.293 & 0.503 & & 0.560 & 0.472 & 0.722 & & $-$ & $-$ & $-$ \\
    StarE & \textbf{0.491} & 0.398 & \textbf{0.648} & & 0.574 & 0.496 & 0.725 & & \textbf{0.349} & \underline{0.271} & \textbf{0.496} \\
    \hline
    ShrinkE & \underline{0.485} & \textbf{0.431} & \underline{0.601} & & \textbf{0.589} &\textbf{ 0.506} & \textbf{0.749} &  & \underline{ 0.345} & \textbf{0.275} & \underline{0.482} \\
    \hline
    \end{tabular}
    }
    \label{tab:main_results}
\end{table*}

\begin{table*}[]
    \centering
    \caption{Link prediction results on WD50K splits with the number in the parentheses denoting the ratio of facts with qualifiers. Baseline results are taken from \citet{DBLP:conf/emnlp/GalkinTMUL20}. }
    \begin{tabular}{lccccccccccc}
    \hline 
    \multirow{2}{*}{Method} & \multicolumn{3}{c}{WD50K (33)} & & \multicolumn{3}{c}{WD50K (66)} & & \multicolumn{3}{c}{WD50K (100)} \\
    \cline {2-4} \cline {6-8} \cline{10-12} & MRR & H@1 & H@10 & & MRR & H@ 1 & H@ 10 & & MRR & H@ 1 & H@ 10 \\
    \hline 
    NaLP-Fix & 0.204 & 0.164 & 0.277 & & 0.334 & 0.284 & 0.423 & & 0.458 & 0.398 & 0.563 \\
    HINGE & 0.253 & 0.190 & 0.372 & & 0.378 & 0.307 & 0.512 & & 0.492 & 0.417 & 0.636 \\
    Transformer & 0.276 & 0.227 & 0.371 & & 0.404 & 0.352 & 0.502 & & 0.562 & 0.499 & 0.677 \\
    StarE & \underline{0.331} & \underline{0.268} & \textbf{0.451} & & \underline{0.481} & \underline{0.420} & \underline{0.594} & & \underline{0.654} & \underline{0.588} & \underline{0.777} \\
    \hline
    ShrinkE & \textbf{0.336} & \textbf{0.272} & \underline{0.449} & & \textbf{0.511} & \textbf{0.422} & \textbf{0.611} & & \textbf{0.695} & \textbf{0.629} & \textbf{0.814}   \\
    \hline
    \end{tabular}
    \label{tab:wd50k-variants}
\end{table*}

\section{Evaluation}

In this section, we evaluate the effectiveness of ShrinkE on hyper-relational link prediction tasks. 
 
\subsection{Experimental Setup}

\paragraph{Datasets.} 
We conduct link prediction experiment on three hyper-relational KGs: JF17K \citep{DBLP:conf/ijcai/WenLMCZ16}, WikiPeople \citep{DBLP:conf/www/GuanJWC19}, and WD50k \citep{DBLP:conf/emnlp/GalkinTMUL20}. 
JF17K is extracted from Freebase while WikiPeople and WD50k are extracted from Wikidata. In WikiPeople and WD50k, only $11.6\%$ and $13.6\%$ of the facts, respectively, contain qualifiers, while the remaining facts contain only triples (after dropping statements containing literals in WikiPeople, only 2.6\% facts contain qualifiers). 
For better comparison, we also consider three splits of WD50K that contain a higher percentage of triples with qualifiers. 
The three splits are WD50K(33), WD50K(66), and WD50K(100), which contain 33\%, 66\%, and 100\% facts with qualifiers, respectively. 
Statistics of the datasets are given in Table \ref{tab:dataset}.
We conjecture that the performance on WikiPeople and WD50k will be dominated by the scores of triple-only facts while the performance on the variants of WD50k will be dominated by the modeling of qualifiers. 
We conjecture that WD50K will be a more challenging benchmark than JF17K and WikiPeople. 
Besides, WD50K still contains only a small percentage (13.6\%) of facts that contain qualifiers.
Since JF17K does not provide a validation set, we split $20\%$ of facts from the training set as the validation set. Details of the three datasets are given in Table \ref{tab:dataset}.

\paragraph{Environments and hyperparameters}
We implement ShrinkE with Python 3.9 and Pytorch 1.11, and train our model on one Nvidia A100 GPU with 40GB of VRAM. 
We use Adam optimizer with a batch size of $128$ and an initial learning rate of $0.0001$.  
For negative sampling, we follow the strategy used in StarE \citep{DBLP:conf/emnlp/GalkinTMUL20} by randomly corrupting the head or tail entity in the primal triple. 
Different from HINGE \citep{DBLP:conf/www/RossoYC20} and NeuInfer \citep{DBLP:conf/acl/GuanJGWC20} that score all potential facts one by one that takes an extremely long time for evaluation, ShrinkE ranks each target answer against all candidates in a single pass and significantly reduces the evaluation time. 
We search the dimensionality from $[50, 100,200,300]$ and the best one is $200$. We set the temperature parameter to be $t=1.0$. We use the label smoothing strategy and set the smoothing rate to be $0.1$. 
We repeat all experiments for $5$ times with different random seeds and report the average values, the error bars are relatively small and are omitted. 
Code is available at \footnote{https://github.com/xiongbo010/ShrinkE}.

\paragraph{Baselines} 
We compare ShrinkE against various models, including m-TransH \citep{DBLP:conf/ijcai/WenLMCZ16}, RAE \citep{DBLP:conf/www/ZhangLMM18}, NaLP-Fix \citep{DBLP:conf/www/RossoYC20}, HINGE \citep{DBLP:conf/www/RossoYC20}, NeuInfer \citep{DBLP:conf/acl/GuanJGWC20}, BoxE \cite{boxE}, Transformer and StarE \citep{DBLP:conf/emnlp/GalkinTMUL20}. Note that we exclude Hy-Transformer \citep{DBLP:journals/corr/abs-2104-08167}, GRAN 
\citep{DBLP:conf/acl/WangWLZ21} and QUAD \cite{DBLP:journals/corr/abs-2208-14322} for comparison because 
1) they are heavily based on StarE and Transformer; and
2) they leverage auxiliary training tasks, which can also be incorporated into our framework and we leave as one future work.

\paragraph{Evaluation} 
We strictly follow the settings of  \citet{DBLP:conf/emnlp/GalkinTMUL20}, where the aim is to predict a missing head/tail entity in a hyper-relational fact. 
We consider the widely used ranking-based metrics for link prediction: mean reciprocal rank (MRR) and H@K (K=1,10). For ranking calculation, we consider the filtered setting by filtering the facts 
existing in the training and validation sets \cite{DBLP:conf/nips/BordesUGWY13}.

\subsection{Main Results and Analysis}

Table \ref{tab:main_results} and Table \ref{tab:wd50k-variants} summarize the performances of all approaches on the six datasets. Overall, ShrinkE achieves either the best or the second-best results against all baselines, showcasing the expressivity and capability of ShrinkE on hyper-relational link prediction. 
In particular, We observe that ShrinkE outperforms all baselines on JF17K and the three variants of WD50K with a high ratio of facts containing qualifiers while achieving highly competitive results on WikiPeople and the original version of WD50K that contain fewer facts with qualifiers. 
Interestingly, we find that the performance gains increase when increasing the ratio of facts containing qualifiers. On WD50K (100) where 100\% facts contain qualifiers, the performance gain of ShrinkE is most significant across all metrics (6.2\%, 6.9\%, and 4.7\% improvements over MRR, H@1, and H@10, respectively). 
We believe this is because that ShrinkE is excellent at modeling qualifiers due to its explicit modeling of inference patterns. 



\begin{table}[]
    \centering
    \resizebox{\linewidth}{!}{
    \begin{tabular}{cccc}
    \hline
    Method & MRR & H@ 1 & H@ 10 \\
    \hline
    ShrinkE (w/o translation) & 0.583 & 0.495 & 0.729 \\
    ShrinkE (w/o rotation) & 0.581 & 0.497 & 0.724 \\
    ShrinkE (w/o shrinking) & 0.571 & 0.490 & 0.711 \\
    \hline
    ShrinkE  & \textbf{0.589} & \textbf{0.506} & \textbf{0.749} \\
    \hline
    \end{tabular}
    }
    \caption{The performance of ShrinkE by removing one relational component on JF17K. }
    \label{tab:ablation_ShrinkE}
\end{table}

\paragraph{Case analysis}
Table \ref{tab:qualifier_implication} shows some examples of qualifier implication pairs recovered by our learned embeddings. Note that exclusions pairs are ubiquitous (i.e., most of the random qualifiers are mutually exclusive) and hence we do not analyze them. We find that some qualifier implications happen when they are about geographic information and involve geographic inclusion such as \emph{Monte Carlo} is in \emph{Monaco}. Interestingly, we find that qualifiers associated with key \emph{owned\_by} imply (\emph{of, voting interest}), and qualifiers with key \emph{emergency phone number} imply (\emph{has\_use, police}) or (\emph{has\_use, \emph{file department}}), which conceptually make sense.

\subsection{Ablations and Parameter Sensitivity}

\paragraph{Impact of relational components} 
To determine the importance of each component in relational modeling, we conduct an ablation study by considering three versions of ShrinkE in which one of the components (translation, rotation, and shrinking) is removed. 
Table \ref{tab:ablation_ShrinkE} shows that the removal of each component of the relational transformation leads to a degradation in performance, validating the importance of each component.
In particular, by removing the qualifier shrinking, which is the main contribution of our framework, the performance reduces 3\% and 5\% in  MRR and H@10, respectively, showcasing the usefulness of modeling qualifiers as shrinking. The removals of translation and rotation both result in around 1\% and 2\% reduction in MRR and H@10, respectively. 

\begin{table}[]
    \centering
    \resizebox{\linewidth}{!}{
    \begin{tabular}{cc}
\hline
body & head \\
\hline
 (residence: Monte Carlo) & (country, Monaco)  \\
 (residence: Belgrade) & (country, Serbia) \\
 (owned\_by: X) & (of, voting interest) \\
 (emergency phone number: Y) & (has\_use, police) \\
 (emergency phone number: Z) & (has\_use, fire department) \\
 (used\_by: software) & (via, operating\_system) \\
\hline
\end{tabular}
}
\caption{Example pairs of qualifiers with implication relations (body $\rightarrow$ head). $X \in $ [\emph{Eric Schmidt}, \emph{Mark Zuckerberg}, \emph{Dustin Moskovitz}, \emph{Larry Page}] denotes a CEO name of a company. $Y \in [112,115,113, \cdots]$ and $Z \in [912,18,192, \cdots]$ are emergency numbers involving \emph{police} and \emph{fire department}, respectively. 
Qualifier exclusion pairs are ubiquitous and are hence omitted. }
    \label{tab:qualifier_implication}
\end{table}

\begin{figure}
    \centering
    \includegraphics[width=0.8\textwidth]{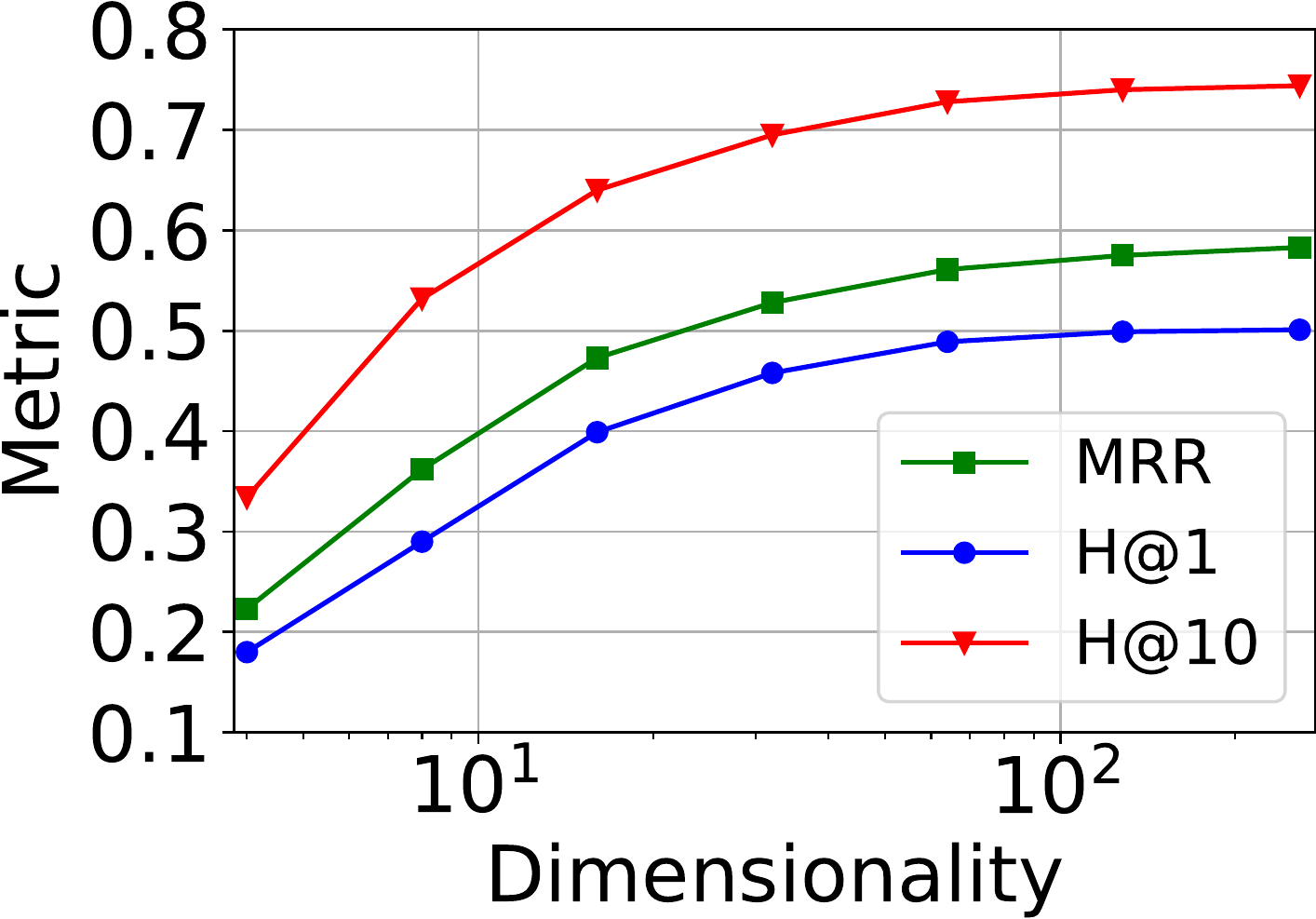}
    \caption{Performance of ShrinkE with different dimensions $d=[4,8,16,32,64,128,256]$ on JF17K. }
    \label{fig:dimension}
\end{figure}

\paragraph{Impact of dimensionality} 
We conduct experiments on JF17K under a varied number of dimensions $d=[4,8,16,32,64,128,256]$. As Fig. \ref{fig:dimension} depicts,  the performance increases when increasing the number of dimensions. 
However, the growth trend gradually flattens with the increase of dimensions and it achieves comparable performance when the dimension is higher than $128$.

\subsection{Discussion}

\paragraph{Comparison with neural network models} 
Heavy neural network models such as GRAN 
\citep{DBLP:conf/acl/WangWLZ21} and QUAD \cite{DBLP:journals/corr/abs-2208-14322} are built on relational GNNs and/or Transformers and require a large number of parameters. In contrast, ShrinkE is a neuro-symbolic model that requires only one MLP layer and a much smaller number of parameters. The logical modelling of ShrinkE makes it more explainable than GNN-based and Transformer-based methods. 

\paragraph{Comparison with other box embeddings in KGs} 
ShrinkE is the first to not only represent hyper-relational facts, but also explicitly model the logical properties of these facts. SrinkE is different from previous box embedding methods \cite{boxE} of KGs in three key modules: 1) our point-to-box transform function modelling triple inference patterns; 2) a new point-to-box distance function; and 3) we introduce box shrinking to model qualifier-level inference patterns. Moreover, we provide a comprehensive theoretical analysis of ShrinkE on modelling various logical properties.

\section{Conclusion}
We present a novel hyper-relational KG embedding model ShrinkE. ShrinkE models a primal triple as a spatio-functional transformation while modeling each qualifier as a shrinking that monotonically narrows down the answer set. We proved that ShrinkE is able to spatially infer core inference patterns at different levels including triple-level, fact-level, and qualifier-level. Experimental results on three benchmarks demonstrate the advantages of ShrinkE in predicting hyper-relational links.  

\section*{Limitations}

Currently, the main goal of ShrinkE is to model inference patterns directly in the embedding space for hyper-relational KGs and we do not explore more advanced training strategies that have recently been proposed. 
For example, recent works \citep{DBLP:journals/corr/abs-2104-08167,DBLP:conf/acl/WangWLZ21,DBLP:journals/corr/abs-2208-14322} have demonstrated that adding auxiliary training tasks, e.g., the task of predicting qualifier entities, can further improve the overall performance.  
We believe such auxiliary training tasks can also benefit ShrinkE and we leave it as future work. 
Another limitation of ShrinkE, though rarely happens, is that when dealing with semantically opaque contexts, the monotonicity assumption might not hold. In that case, we need ad-hoc solutions. One simple way is to explicitly distinguish semantically transparent and semantically opaque contexts. 

\section*{Ethics Statement}
The authors declare that they have no conflicts of interest. 
This article does not contain any studies involving business data and personal information. 
Our experimentation does not involve any ethical concerns. 
However, similar to other models, when deploying our link prediction model to real-world applications such as online recommendation systems, the prediction might be biased or unfair to some ethic/gender groups. We advise researchers in the community to look into bias \cite{bourli2020bias} and fairness \cite{fu2020fairness} in KGs. 

\section*{Acknowledgement}
The authors thank the International Max Planck Research School for Intelligent Systems (IMPRS-IS) for supporting Bo Xiong. 
Bo Xiong is funded by the European Union’s Horizon 2020 research and innovation programme under the Marie Skłodowska-Curie grant agreement No: 860801. Mojtaba Nayyeri is funded by the German Federal Ministry for Economic Affairs and Climate Action under Grant Agreement Number 01MK20008F (Service-Meister). This research was partially funded by the Ministry of Science, Research, and the Arts (MWK) Baden-Württemberg, Germany, within the Artificial Intelligence Software Academy (AISA) and the German Research Foundation (DFG) via grant agreement number STA 572/18-1 (Open Argument Mining). We acknowledge the support by the Stuttgart Center for Simulation Science (SimTech).

\bibliographystyle{acl_natbib}
\bibliography{anthology}

\newpage
\appendix
\label{sec:appendix}

\section{Supplemental Related Works}
We survey some supplemental related work on binary relational KG embeddings and geometric relational embeddings. 

\paragraph{Binary relational KG embeddings}
Most of the existing KG embedding methods consider binary relational KGs where each fact is represented in the form of triple $(h,r,t)$. Prominent examples include the \emph{additive} (or \emph{translational}) family such as TransE \cite{DBLP:conf/nips/BordesUGWY13} that models each fact as a translation $\mathbf{s}+\mathbf{r} \approx \mathbf{o}$, and the \emph{multiplicative} (or \emph{bilinear}) family such as RESCAL \cite{DBLP:conf/icml/NickelTK11} that models the relation between two entities as a bilinear interaction $<\mathbf{h},\mathbf{r},\mathbf{t}>$. 
Many other works have been proposed to enhance the translational and bilinear models such as modeling relational mapping properties (e.g., one-to-many and many-to-many) \cite{DBLP:conf/aaai/WangZFC14}, modeling inference patterns (e.g., symmetry and composition) \cite{DBLP:conf/icml/TrouillonWRGB16, DBLP:conf/iclr/SunDNT19}, and modeling complex graph structures (e.g., hierarchies and cycles) \cite{DBLP:conf/acl/ChamiWJSRR20, DBLP:conf/kdd/XiongZNXP0S22} to name a few. 


\paragraph{Geometric relational embeddings} Our work is closely related to geometric relational embeddings. See \cite{DBLP:journals/corr/abs-2304-11949} for a systematic survey. Geometric relational embeddings encode real-world relational knowledge by geometric objects such as convex regions like $n$-balls \citep{kulmanov2019embeddings}, convex cones \citep{ DBLP:conf/nips/ZhangWCJW21,DBLP:journals/corr/abs-2303-11858}, axis-parallel boxes \citep{DBLP:conf/acl/McCallumVLM18,DBLP:journals/corr/abs-2201-09919,DBLP:conf/iclr/RenHL20} and non-Euclidean manifold components \cite{DBLP:conf/nips/XiongCNS22}. A key advantage of these geometric embeddings is that they nicely model the set-theoretic semantics that can be used to capture logical rules of KGs~\cite{boxE}, ontological axioms \citep{kulmanov2019embeddings, DBLP:journals/corr/abs-2201-09919}, transitive closure~\citep{DBLP:conf/acl/McCallumVLM18}, and logical query for multi-hop reasoning~\citep{DBLP:conf/iclr/RenHL20}. 
Different from all previous work, ShrinkE is the first geometric embedding that aims at modeling inference patterns for hyper-relational KGs.

\section{Proof of propositions}
\label{app:proof}

\begin{proposition}
Given any two facts $\mathcal{F}_1=\left(\mathcal{T},\mathcal{Q}_1\right)$ and $\mathcal{F}_2=\left(\mathcal{T},\mathcal{Q}_2\right)$ where $\mathcal{Q}_2 \subseteq \mathcal{Q}_1$, i.e., $\mathcal{F}_2$ is a partial fact of $\mathcal{F}_1$, the output of the scoring function $f(\cdot)$ of ShrinkE satisfy the constraint $f(\mathcal{F}_2) \geq f(\mathcal{F}_1)$, which implies Eq.(\ref{eq:momonotonicity}). 
\end{proposition}

\begin{proof}
We first prove that the resulting box of $\mathcal{F}_2$ subsumes the resulting box of $\mathcal{F}_2$.  Since the primal triple of $\mathcal{F}_1$ and $\mathcal{F}_2$ are the same (let assume it is $\mathcal{T}=(h,r,t)$ ), the spanned boxes of the two facts are $\mathcal{H}_r(\mathbf{e}_h)$. Since $\mathcal{Q}_2 \subseteq \mathcal{Q}_1$, the final shrunken box of $\mathcal{F}_1$ must be a subset of the shrunken box of $\mathcal{F}_2$.  Hence, we have,
\begin{equation}
    \Box_{\mathcal{F}_2} \subseteq \Box_{\mathcal{F}_1}. 
\end{equation}
Given the tail entity $t$ whose embedding is denoted by $\mathbf{e}_t$,
we consider three cases of its position. 

1) If $\mathbf{e}_t$ is inside the small box $\Box_{\mathcal{F}_2}$, then $\mathbf{e}_t$ must also be inside $\Box_{\mathcal{F}_1}$ since $\Box_{\mathcal{F}_2} \subseteq \Box_{\mathcal{F}_1}$. Note that our point-to-box function is monotonically increasing w.r.t. the increase of distance from the tail point to the center of box. Hence, we will have $D(\mathbf{e}, \Box_{\mathcal{F}_2}) \geq D(\mathbf{e}, \Box_{\mathcal{F}_1})$, implying $f(\mathcal{F}_2) \geq f(\mathcal{F}_1)$. 

2) If $\mathbf{e}_t$ is outside the small box $\Box_{\mathcal{F}_2}$ but inside in the larger $\Box_{\mathcal{F}_1}$, according to the definition of the point-to-box distance function, we immediately have $D(\mathbf{e}, \Box_{\mathcal{F}_2}) \geq D(\mathbf{e}, \Box_{\mathcal{F}_1})$, implying $f(\mathcal{F}_2) \geq f(\mathcal{F}_1)$. 

3) If $\mathbf{e}_t$ is outside the larger box $\Box_{\mathcal{F}_1}$,, then $\mathbf{e}_t$ must also be outside $\Box_{\mathcal{F}_2}$ since $\Box_{\mathcal{F}_2} \subseteq \Box_{\mathcal{F}_1}$. Note that our point-to-box function is monotonically decreasing w.r.t. the increase of volume of box. Hence, we will have $D(\mathbf{e}, \Box_{\mathcal{F}_2}) \geq D(\mathbf{e}, \Box_{\mathcal{F}_1})$, implying $f(\mathcal{F}_2) \geq f(\mathcal{F}_1)$. 
\end{proof}

\begin{proposition}
    ShrinkE is able to infer hyper-relational symmetry, anti-symmetry, inversion, composition, hierarchy, intersection, and exclusion.  
\end{proposition}

We first prove that ShrinkE is able to infer symmetry, anti-symmetry, inversion, and composition. For the sake of proof, we assume $\theta_r \in [-\pi,\pi)$. We prove them by proving Lemma B.1-4 one by one.  

\begin{lemma}[Symmetry]
Let $r$ be a symmetric relation such that for each triple $(e_h, r, e_t)$, its symmetric triple $(e_t, r, e_h)$ also holds. This symmetric property of $r$ can be modeled by ShrinkE.
\end{lemma}
\begin{proof}
If $r$ is a symmetric relation, by taking the $\mathbf{\boldsymbol\delta}_r=\mathbf{0}$, $\mathbf{b}_r=\mathbf{0}$, and $\mathbf{\Theta}_r=\operatorname{diag}\left(\mathbf{G}\left(\mathbf{\theta}_{r, 1}\right), \ldots, \mathbf{G}\left(\mathbf{\theta}_{r, \frac{d}{2}}\right)\right)$, where $\mathbf{G}(\theta)$ is a $2\times2$ diagonal matrix, we have
\begin{align}
\begin{split}
    & \mathbf{e}_{h} = f_{r}\left(\mathbf{e}_{t}\right) =  \mathbf{\Theta}_r \mathbf{e}_{t}, \ 
   \mathbf{e}_{t} = f_{r}\left(\mathbf{e}_{h}\right) =  \mathbf{\Theta}_r \mathbf{e}_{h} \\
   & \Rightarrow \mathbf{\Theta}_r^2 = \mathbf{I} \nonumber
\end{split}
\end{align}
which holds true when $\mathbf{\theta}_{r,i}=\mathbf{0}$ or $\mathbf{\theta}_{r,i}=-\mathbf{\pi}$ for $i =  1,\cdots,\frac{d}{2}$.
\end{proof}

\begin{lemma}[Anti-symmetry]
Let $r$ be an anti-symmetric relation such that for each triple $(e_h, r, e_t)$, its symmetric triple $(e_t, r, e_h)$ is not true. This anti-symmetric property of $r$ can be modeled by ShrinkE.
\end{lemma}

\begin{proof}
If $r$ is a anti-symmetric relation, by taking the $\mathbf{\boldsymbol\delta}_r=\mathbf{0}$, $\mathbf{b}_r=\mathbf{0}$, and $\mathbf{\Theta}_r=\operatorname{diag}\left(\mathbf{G}\left(\mathbf{\theta}_{r, 1}\right), \ldots, \mathbf{G}\left(\mathbf{\theta}_{r, \frac{d}{2}}\right)\right)$, where $\mathbf{G}(\theta)$ is a $2\times2$ diagonal matrix, we have
\begin{align}
\begin{split}
    & \mathbf{e}_{h} \neq f_{r}\left(\mathbf{e}_{t}\right) =  \mathbf{\Theta}_r \mathbf{e}_{t}, \ 
   \mathbf{e}_{t} = f_{r}\left(\mathbf{e}_{h}\right) =  \mathbf{\Theta}_r \mathbf{e}_{h} \\
   & \Rightarrow \mathbf{\Theta}_r^2 \neq \mathbf{I} \nonumber
\end{split}
\end{align}
which holds true when $\mathbf{\theta}_{r,i} \neq \mathbf{0}$ or $\mathbf{\theta}_{r,i} \neq -\mathbf{\pi}$ for $i =  1,\cdots,\frac{d}{2}$.
\end{proof}

\begin{lemma}[Inversion]
Let $r_1$ and $r_2$ be inverse relations such that for each triple $(e_h, r_1, e_t)$, its inverse triple $(e_t, r_2, e_h)$ is also true. This inverse property of $r_1$ and $r_2$ can be modeled by ShrinkE.
\end{lemma}
\begin{proof}
If $r_1$ and $r_2$ are inverse relations, by taking the $\mathbf{\boldsymbol\delta}_r=\mathbf{0}$, $\mathbf{b}_r=\mathbf{0}$, and $\mathbf{\Theta}_r=\operatorname{diag}\left(\mathbf{G}\left(\mathbf{\theta}_{r, 1}\right), \ldots, \mathbf{G}\left(\mathbf{\theta}_{r, \frac{d}{2}}\right)\right)$, where $\mathbf{G}(\theta)$ is a $2\times2$ diagonal matrix, we have
\begin{equation}
\begin{split}
    &\mathbf{e}_{t} = f_{r_1}\left(\mathbf{e}_{h}\right) =  \Theta_{r_1} \mathbf{e}_{h}, \ 
   \mathbf{e}_{h} = f_{r_2}\left(\mathbf{e}_{t}\right) =  \Theta_{r_2} \mathbf{e}_{h} \\
   & \Rightarrow \Theta_{r_1} \Theta_{r_2} = \mathbf{I} \nonumber
\end{split}
\end{equation}
which holds true when for $\theta_{r_1,i}{r_1}+\theta_{r_2,i}=0$ for $i =  1,\cdots,\frac{d}{2}$.
\end{proof}

\begin{lemma}[Composition]
Let relation $r_1$ be composed of $r_2$ and $r_3$ such that triple $(e_1, r_1, e_3)$ exists when $(e_1, r_2, e_2)$ and $(e_2, r_3, e_3)$ exist. This composition property can be modeled by ShrinkE.
\end{lemma}
\begin{proof}
If $r_1$ is composed of $r_2$ and $r_3$, by taking the $\mathbf{\boldsymbol\delta}_r=\mathbf{0}$, $\mathbf{b}_r=\mathbf{0}$, and $\mathbf{\Theta}_r=\operatorname{diag}\left(\mathbf{G}\left(\mathbf{\theta}_{r, 1}\right), \ldots, \mathbf{G}\left(\mathbf{\theta}_{r, \frac{d}{2}}\right)\right)$, where $\mathbf{G}(\theta)$ is a $2\times2$ diagonal matrix, we have
\begin{equation}
\begin{split}
    &\mathbf{e}_{3} = f_{r_1}\left(\mathbf{e}_{1}\right) =  \Theta_{r_1} \mathbf{e}_{1}, \ 
   \mathbf{e}_{2} = f_{r_2}\left(\mathbf{e}_{1}\right) =  \Theta_{r_2} \mathbf{e}_{1}, \\
   &\mathbf{e}_{3} = f_{r_3}\left(\mathbf{e}_{2}\right) =  \Theta_{r_3} \mathbf{e}_{2} \
   \Rightarrow \Theta_{r_1} = \Theta_{r_2}\Theta_{r_3} \nonumber
\end{split}
\end{equation}
which holds true when $\theta_{r_1,i}=\theta_{r_2,i}+\theta_{r_3,i}$ or $\theta_{r_1,i}=\theta_{r_2,i}+\theta_{r_3,i}+2\pi$ or $\theta_{r_1,i}=\theta_{r_2,i}+\theta_{r_3,i}-2\pi$ for $i =  1,\cdots,\frac{d}{2}$.
\end{proof}

We now prove that ShrinkE is able to infer relation implication, exclusion and intersection.

\begin{lemma}[Relation implication]
Let $r_1 \rightarrow r_2$ form a hierarchy such that for each triple $(e_h, r_1, e_t)$, $(e_h, r_2, e_t)$ also holds. This hierarchy property $r_1 \rightarrow r_2$ can be modeled by ShrinkE.
\end{lemma}
\begin{proof}
If $r_1 \rightarrow r_2$, by taking $\mathcal{T}_{r_1} =\mathcal{T}_{r_2}$, i.e.,  $\boldsymbol\delta_{r_1}=\boldsymbol\delta_{r_2}$ and $\Theta_{r_1}=\Theta_{r_2}$, we have,
$(e_h, r_1, e_t) \rightarrow (e_h, r_2, e_t)$ implies that the spanning box of query $(e_h, r_1, x?)$ is subsumed by the spanning box of query $(e_h, r_2, x?)$. i.e., 
$\Box( \mathcal{H}_{r_1}(e_h)-\sigma(\boldsymbol\delta_{r_1}), \mathcal{H}_{r_1}(e_h)+\sigma(\boldsymbol\delta_{r_1}) ) 
\subseteq \Box( \mathcal{H}_{r_1}(e_h)-\sigma(\boldsymbol\delta_{r_2}), \mathcal{H}_{r_1}(e_h)+\sigma(\boldsymbol\delta_{r_2}))$,
which holds true when $\boldsymbol\delta_{r_1} \leq  \boldsymbol\delta_{r_2}$.
\label{lem:hierarchy}
\end{proof}

\begin{lemma}[Relation exclusion]
Let $r_1, r_2$ be mutually exclusive, that is, $(e_h, r_1, e_t)$, $(e_h, r_2, e_t)$ can not be simultaneously hold. This mutual exclusion property $r_1 \wedge r_2 \rightarrow \bot$ can be modeled by ShrinkE.
\end{lemma}
\begin{proof}
If $r_1 \wedge r_2 \rightarrow \bot$, we have
$(e_h, r_1, e_t) \wedge (e_h, r_2, e_t) \rightarrow \bot$, which implies that the spanning box of query $(e_h, r_1, x?)$ and the spanning box of query $(e_h, r_2, x?)$ are mutually exclusive, i.e., 
$\Box( \mathcal{H}_{r_1}(e_h)-\sigma(\boldsymbol\delta_{r_1}), \mathcal{H}_{r_1}(e_h)+\sigma(\boldsymbol\delta_{r_1}) ) 
\cap \Box( \mathcal{H}_{r_1}(e_h)-\sigma(\boldsymbol\delta_{r_2}), \mathcal{H}_{r_1}(e_h)+\sigma(\boldsymbol\delta_{r_2})) \rightarrow \bot$
\end{proof}

\begin{lemma}[Relation intersection]
Let $r_3$ be a intersection of $r_1, r_2$, that is, if $(e_h, r_1, e_t)$ and $(e_h, r_2, e_t)$ hold, then  $(e_h, r_3, e_t)$ also holds. This intersection property $r_1 \wedge r_2 \rightarrow r_3$ can be modeled by ShrinkE.
\end{lemma}
\begin{proof}
Note that box is closed under intersection and this property can be view as a combination of two pairs of relation implication. Hence, the proof is similar to the proof of Lemma \ref{lem:hierarchy}. 
\end{proof}

\begin{proposition}
ShrinkE is able to infer qualifier implication, mutual exclusion, and intersection.
\end{proposition}
\begin{proof}
Since each qualifier is associated with a box, the implication and mutual exclusion relationships between qualifiers can be modeled by their geometric relationships, i.e., box entailment and box disjointedness, respectively, between their corresponding boxes.  
Qualifier intersection can be modeled by enforcing the box of one qualifier to be inside the intersection of the boxes of another two qualifiers.  
\end{proof}

\end{document}